\documentclass[onecolumn,draftfootnocls]{IEEEtran}

\usepackage{amsmath,amsfonts,bm,amsthm}

\newtheorem{theorem}{Theorem}
\newtheorem{defn}{Definition}
\newtheorem{assumption}{Assumption}
\newtheorem{lemma}[theorem]{Lemma}

\newtheorem{remark}{Remark}

\def\1{\bm{1}}

\newcommand{\valid}{\textbf{$\boldsymbol{\top}$}}
\newcommand{\invalid}{\textbf{$\boldsymbol{\perp}$}}

\usepackage{pgffor}
\foreach \x in {a,...,z}{%
\expandafter\xdef\csname vec\x \endcsname{\noexpand\ensuremath{\noexpand\bm{\x}}}
}

\foreach \x in {A,...,Z}{%
\expandafter\xdef\csname vec\x \endcsname{\noexpand\ensuremath{\noexpand\bm{\x}}}
}

\foreach \x in {A,...,Z}{%
\expandafter\xdef\csname c\x \endcsname{\noexpand\ensuremath{\noexpand\mathcal{\x}}}
}

\foreach \x in {A,...,Z}{%
\expandafter\xdef\csname bb\x \endcsname{\noexpand\ensuremath{\noexpand\mathbb{\x}}}
}

\def\vecGamma{{\bm{\Gamma}}}

\def\vs{{\bm{s}}}

\def\vx{{\bm{x}}}
\def\vy{{\bm{y}}}

\def\mD{{\bm{D}}}

\def\mG{{\bm{G}}}

\def\mS{{\bm{S}}}

\def\mW{{\bm{W}}}
\def\mX{{\bm{X}}}
\def\mY{{\bm{Y}}}
\def\mZ{{\bm{Z}}}

\DeclareMathAlphabet{\mathsfit}{\encodingdefault}{\sfdefault}{m}{sl}
\SetMathAlphabet{\mathsfit}{bold}{\encodingdefault}{\sfdefault}{bx}{n}
\newcommand{\tens}[1]{\bm{\mathsfit{#1}}}

\def\tD{{\tens{D}}}

\newcommand{\ltwo}[1]{\left\lVert #1\right\rVert_2}

\DeclareMathOperator*{\argmin}{arg\,min}

\newcommand{\ie}{\emph{i.e.}}

\usepackage{times}
\usepackage{url}
\usepackage{latexsym}
\usepackage{graphicx}
\usepackage{makecell}
\usepackage{array}
\usepackage[pass]{geometry} 

\newcommand{\subparagraph}{}

\usepackage{amsmath,amssymb,amsfonts,amsthm,xspace,tikz}
\usepackage[ruled, lined, noend, longend]{algorithm2e}
\SetAlgorithmName{Protocol}{protocol}{List of Protocols}

\usepackage{graphicx}
\usepackage{textcomp}
\usepackage{pdfpages}
\usepackage{pdflscape}
\usepackage{comment}
\usepackage{mathtools}
\mathtoolsset{showonlyrefs=true}
\usetikzlibrary{arrows,shapes.geometric}
\allowdisplaybreaks
\definecolor{pastelRed}{RGB}{255,105,97}
\definecolor{pastelGreen}{RGB}{119,221,119}
\definecolor{pastelBlue}{RGB}{79,195,247}
\definecolor{pastelOrange}{RGB}{255,179,71}

\definecolor{darkred}{RGB}{175, 0, 0}
\definecolor{darkgreen}{RGB}{0, 128, 0}
\definecolor{darkblue}{RGB}{0, 0, 175}
\definecolor{darkorange}{RGB}{255, 128, 0}

\newcommand{\protocol}{~{\scshape Valid}~}

\begin{document}

\title{\protocol: a Validated Algorithm for Learning in Decentralized Networks with Possible Adversarial Presence}
\author{\IEEEauthorblockN{Mayank Bakshi\IEEEauthorrefmark{1}\thanks{This material is based upon work supported by the National Science Foundation under Grant No. CCF-1908725, CCF-2107526, CCF-2107370, and CCF-2107488. Author emails: mayank.bakshi@ieee.org, sg273@njit.edu, yauhen.yakimenka@njit.edu, beemera@uwec.edu, okosut@asu.edu, jkliewer@njit.edu.}\ \  Sara Ghasvarianjahromi\IEEEauthorrefmark{2}\ \ Yauhen Yakimenka\IEEEauthorrefmark{2}\ \ Allison Beemer\IEEEauthorrefmark{3}\ \  Oliver Kosut\IEEEauthorrefmark{1}\ \  J\"org Kliewer\IEEEauthorrefmark{2}} 
\IEEEauthorblockA{\IEEEauthorrefmark{1}Arizona State University\qquad   \IEEEauthorrefmark{2}New Jersey Institute of Technology\qquad   \IEEEauthorrefmark{3}University of Wisconsin-Eau Claire}}


\date{}
\maketitle
\vspace{-0.25em}
\begin{abstract}
We introduce the paradigm of \emph{validated decentralized learning} for undirected networks with heterogeneous data and possible adversarial infiltration.  We require \emph{(a)} convergence to a global empirical loss minimizer when adversaries are absent, and \emph{(b)} either detection of adversarial presence or convergence to an admissible consensus model in their presence. This contrasts sharply with the traditional byzantine-robustness requirement of convergence to an admissible consensus irrespective of the adversarial configuration. To this end, we propose the \protocol protocol which, to the best of our knowledge, is the first to achieve a validated learning guarantee. Moreover, \protocol offers an $O(1/T)$ convergence rate (under pertinent regularity assumptions), and computational and communication complexities comparable to non-adversarial distributed stochastic gradient descent. Remarkably, \protocol  retains optimal performance metrics in adversary-free environments, sidestepping the robustness penalties observed in prior byzantine-robust methods. {A distinctive aspect of our study is a heterogeneity metric based on the norms of individual agents' gradients  computed at the global empirical loss minimizer. This not only provides a natural statistic for detecting significant byzantine disruptions but also allows us to prove the optimality of \protocol in wide generality.  
 Lastly, our numerical results reveal that,  in the absence of  adversaries, \protocol converges faster than state-of-the-art byzantine robust algorithms, while when adversaries are present, \protocol terminates with each honest agent either converging to  an admissible consensus or declaring adversarial presence in the network.}
\end{abstract}

\section{Introduction}\vspace{-0.25em}
Machine learning is increasingly reliant on data from a variety of distributed sources. As such, it may be difficult to ensure that the data which originates from these sources is trustworthy. Thus, there is a need to develop distributed and decentralized learning strategies that can respond to bad or even malicious data. However, worst-case or \emph{Byzantine} resilience is an extremely strong requirement, that performance be maintained if a malicious adversary controls a subset of the processing nodes and takes any conceivable action. In practice, an adversary launching such an attack against a learning process requires tremendous resources which may not be worth the cost to influence the learned model. Thus, even though malicious adversaries are a threat, for the vast majority of the time, they are not present. An algorithm that maintains Byzantine robustness necessarily sacrifices performance when no adversaries are present. Instead, we seek a middle ground in which there is no performance loss at all when adversaries are not present, but they can still be detected if they are.

We consider decentralized learning in a network of agents --- in contrast to federated learning~\cite{McMahMRHA:17,KonecMYRSB:16}, there is no trusted central processing entity. Each agent in the network may communicate with its neighbors in a topology graph, and it has access to data samples from a certain distribution. The goal is to learn a consensus model at each agent that minimizes a total loss function from combining all agents' individual data. Rather than Byzantine robustness, we require the weaker assumption of \emph{validation}. In particular, when no adversaries are present, there should be no performance degradation compared to the completely honest setting. When one or more adversaries are present, one of two outcomes may occur for each honest (non-adversarial) agent: either (i) the model is learned with minimal deviation from the completely honest setting, or (ii) the agent signals an alarm, indicating that it has detected the presence of the adversary. In the latter case, the honest agents signaling the alarm do not converge on a model; however, once they signal the alarm, they may fall back to a more conservative Byzantine resilient algorithm, or even abandon the learning task in favor of ridding the network of adversaries.

\noindent{\em Outline of the paper:} We formally introduce our problem setting in Section~\ref{sec:prelims}. Next, in Sections~\ref{sec:mainresults} and~\ref{sec:validprotocol}, we describe our main results and give an overview of the~\protocol protocol. 
Finally, in Section~\ref{sec:experiments}, we give experimental results to verify the performance of~\protocol on practical datasets and compare it against relevant benchmarks. In the interest of space, we present the detailed descriptions of our protocols and the proof arguments in the appendices.  We begin with a review of related work below.

\noindent{\em Related work:} 
 Early works such as the Byzantine Generals Problem~\cite{PeaseSL:80,LampoSP:82a,Dolev:82,DolevS:83} laid the groundwork for Byzantine fault tolerance. \cite{DolevRS:90} introduced 'early stopping' algorithms allowing adaptivity to the actual number of Byzantine faults. Along this direction, \cite{CastrL:99,AbrahD:15,MartiA:05} study algorithms that optimize performance in fault-free scenarios while retaining worst-case robustness. Cachin \emph{et al.}'s work~\cite{CachiKPS:01a} on broadcast primitives for asynchronous Byzantine consensus is particularly relevant to our validation focus.

Work in decentralized learning may be traced back to distributed optimization methods 
(\emph{c.f.},~\cite{TsitsBA:86,NedicO:09}). Subsequent work has examined both deterministic and stochastic optimization techniques in this context~\cite{RamNV:09,NedicOP:10,LianZZHZL:17}. See~\cite{NedicPSS:18,YangYWYWMHWLJ:19} for a comprehensive survey of various methods in this area. 

Byzantine-resilience in distributed optimization setups was first examined in the context of distributed estimation~\cite{VempaTV:13,KailkHBV:15,LaiRV:16,ChenKM:18} and Federated Learning~\cite{BlancEGS:17,AlistAL:18,YinCKB:18,XieKG:19,XieKG:20}. The issue of Byzantine adversaries on distributed optimization was first examined by Su~\emph{et al.}~\cite{SuV:15} and has led to a substantial area of research,~\emph{c.f.}~\cite{SuV:15a,SuV:15b,XuLL:18,MhamdGR:18,SundaG:19,KuwarXS:20,SuV:21,YeminNGG:22}.  We refer the reader to~\cite{YangGB:20} for a comprehensive survey of these areas. In the context of decentralized learning, two broad approaches have emerged -- \emph{screening-based protocols} that involve outlier-robust aggregation of other agents' models by each agent~\cite{El-MhFGGHR:2021,YangB:19,FangYB:22,HouWWHHG:22} and \emph{performance-based protocols} that involve using each agent's local data to detect and eliminate outliers~\cite{GuoZYXMXL:22,ElkorPA:22}. Our validation requirement is similar to that in~\cite{ChenKM:18} that examines the problem of distributed estimation with an option to raise a ``flag'' if adversaries are detected.

\section{Preliminaries}\label{sec:prelims}

\subsection{Learning Problem} \vspace{-0.25em} We consider the problem of distributed empirical risk minimization across a set of agents $\cV$. Each agent \(v\in\cV\) observes data sequentially, drawn independently and identically according to its \emph{local distribution} \(P_v\in\cP(\cD)\), where the class of data distributions $\cP(\cD)$ under consideration is known {\em a priori}. We employ a loss function \(f: \mathbb{R}^d \times \mathcal{D} \to \mathbb{R}^+\) that evaluates the goodness-of-fit between the model vector \(\vecx \in \mathbb{R}^d\) and the data \(D\in\cD\) with smaller loss values indicating a better fit. 

For any subset \(\mathcal{U}=\{u_1,u_2,\ldots,u_m\}\) of \(\mathcal{V}\), the joint data distribution is a product distribution \(P_{\mathcal{U}} = P_{u_1} \times P_{u_2} \times \ldots \times P_{u_m}\in\left(\cP(\cD)\right)^{|\cU|}\). The joint distribution of all agents' data \(P_{\mathcal{V}}\) is termed the \emph{global data distribution}. 
The ideal learning goal\footnote{We refer to this as the ``ideal'' learning goal as finding $\vecx^*(P_{\cV})$ (and hence, achieving $f^*(P_{\cV})$) is, in general, too optimistic in the presence of adversaries. We elaborate on this further in Section~\ref{sec:admissiblemodel}} is to find a global empirical loss minimizer
\begin{equation}
	\vecx^*(P_{\mathcal{V}}) = \argmin_{\vecx \in \mathbb{R}^d} \frac{1}{|\mathcal{V}|} \sum_{v \in \mathcal{V}} \mathbb{E}_{D \sim P_v} f(\vecx, D),\label{eq:xstar}
\end{equation}
that attains the loss value
	$f^*(P_{\mathcal{V}}) = \frac{1}{|\mathcal{V}|} \sum_{v \in \mathcal{V}} \mathbb{E}_{D \sim P_v} f(\vecx^*(P_{\mathcal{V}}), D)$.

\subsection{Network, Protocols, and Adversaries}\vspace{-0.25em}
\begin{figure}[t]
\centering
 
        \resizebox{0.5\textwidth}{!}{\begin{tikzpicture}[scale=1.5,>=stealth]
            \node[circle, draw, darkgreen, fill=white!90!green, minimum size=0.5cm] (v) at (0,0) {$v$};
            \node[circle, draw=gray, fill=white, minimum size=0.5cm] (a) at (1.5,-0.5) {\color{gray}$a$};
            \node[circle, draw=gray, fill=white, minimum size=0.5cm] (b) at (-0.7,1.2) {\color{gray}$b$};
            \node[circle, draw=red, fill=red!20!white, minimum size=0.5cm] (c) at (-1,-1) {$c$};
           
 			\node[rectangle, draw, black!50!darkgreen, dotted, fill=white!99!green, minimum size=0.5cm,align=left] (vin) at (3,1) { \small \color{black!50!darkgreen}  Sent messages: $(\vecM_{vu}^{(\tau)}:u\in\cN(v),\tau\in[t-1])$\\ Recvd. messages: $(\vecM_{uv}^{(\tau)}:u\in\cN(v),\tau\in[t-1])$\\Data mini-batches: $(\mD^{(\tau)}_v:\tau\in[t])$\\Graph $\cG$, protocol $\Pi$};
           \node[rectangle, draw=red, thick, dotted, fill=red!2!white, minimum size=0.5cm,align=left] (cin) at (2,-1.7) { \small All messages: $(\vecM_{uw}^{(\tau)}:(u,w)\in\cE,\tau\in[t-1])$\\ All data mini-batches: $(\mD^{(\tau)}_u:u\in\cV,\tau\in[t])$\\ Graph $\cG$, protocol $\Pi$, global data distribution $P_{\cV}$};
            \draw[->, gray, dashed] (a) -- ++(0.5,-0.2);
            \draw[->, gray, dashed] (b) -- ++(-0.2,0.5);
            \draw[->, dashed,red ] (c) -- ++(-0.2,-0.5);
    
            \draw[<-,gray,  dashed] (a) -- ++(0.5,0.2);
            \draw[<-,gray,  dashed] (b) -- ++(-0.5,0.2);
            \draw[<-,gray,  dashed] (c) -- ++(-0.5,-0.2);
    
            
            \draw[->, thick, black,dotted,darkgreen] (vin) to[bend right=10] node[font=\small] {} (v);
            	\draw[->, thick,red,dotted] (cin.north west) to[bend right=7] node[ above, sloped, font=\small] {} (c);
            \draw[->, thick, black] (v) to[bend left=15] node[midway, above, sloped, font=\small] {$\vecM_{va}^{(t)}$} (a);
            
            \draw[->, thick, black] (a) to[bend left=15] node[midway, below, sloped, font=\small] {$\vecM_{av}^{(t-1)}$} (v);
            
            \draw[->, thick, black] (v) to[bend left=15] node[midway, below, sloped, font=\small] {$\vecM_{vb}^{(t)}$} (b);
            \draw[->, thick, black] (b) to[bend left=15] node[midway, above, sloped, font=\small] {$\vecM_{bv}^{(t-1)}$} (v);

            \draw[->, thick, black] (v) to[bend left=15] node[midway, below, sloped, font=\small] {$\vecM_{vc}^{(t)}$} (c);
            \draw[->, thick, red!90!white] (c) to[bend left=15] node[midway, above, sloped, font=\small] {$\vecM_{cv}^{(t-1)}$} (v);

        \end{tikzpicture}}
        \caption{This figure depicts the $t$-th round of a validated learning protocol (for $t=1,2,\ldots, T$). Agent $v$ is an honest agent and agent $c$ is a Byzantine agent. The rectangles next to agents $v$ and agent $c$ denote the available information to them at the beginning of the $t$-th round.}\label{fig:problem}
        \end{figure}
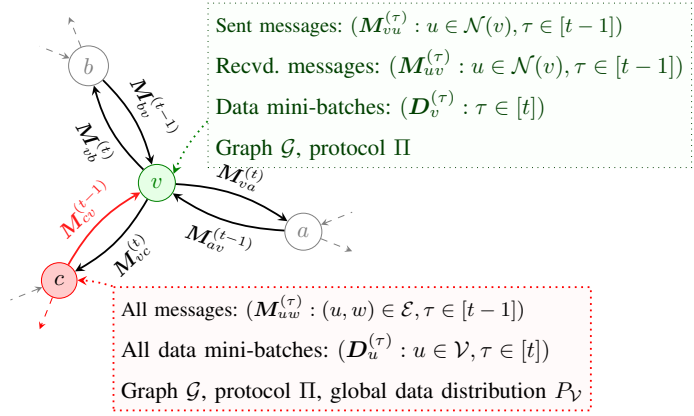
        \begin{figure}[t]
        	\centering
 
        \resizebox{0.6\textwidth}{!}{\begin{tikzpicture}[scale=1.5,>=stealth]
            \node[circle, draw, fill=white, minimum size=0.5cm] (v) at (0,0) {$v$};
            \node[dotted,rectangle,draw,fill=white, minimum size=0.5cm,align=left] (vinput) at (-1.9,0) {\small $(\vecM_{vu}^{(\tau)}:u\in\cN(v),\tau\in[T])$\\$(\vecM_{uv}^{(\tau)}:u\in\cN(v),\tau\in[T])$\\$(\mD^{(\tau)}_v:\tau\in[T])$};
            	\coordinate (inter) at (0.5,0) {};
            	\node (out1) at (1.4,0.5) {};
            	\node (out2) at (1.4,-0.5) {};
            	\node (or) at (2.1,0) {OR};

            	\node[dotted,rectangle,fill=white, minimum size=0.5cm,align=center] (voutput1) at (2.9,0.5) {\small $S_v=\valid$ and final model $\hat{\vecx}_v\in\bbR^d$};
          \node[dotted,rectangle,fill=white, minimum size=0.5cm,align=center] (voutput2) at (1.8,-0.5) {\small $S_v=\invalid$};
           \draw[->, dotted] (vinput) -- (v);
           \draw[->, thick,dashed] (inter) to node[ above, sloped, font=\small] {Validate} (out1);
           \draw[->, thick,dashed] (inter) to node[ below, sloped, font=\small] {Invalidate} (out2);
			\draw[-, thick,dashed] (v) -- (inter);

      \end{tikzpicture}}
        \caption{At the conclusion of the $T$-th round, agent $v$ either sets its validation state $S_v$ to \valid\ to indicate a valid consensus or to \invalid\ to declare Byzantine presence.}\label{fig:goal}
        \end{figure}
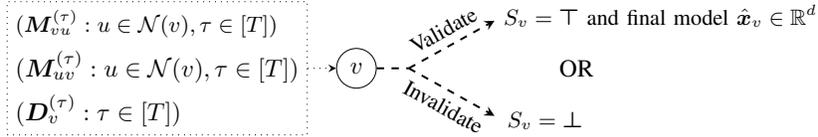

\subsubsection{Agent Graph} Agents are interconnected via an undirected graph \(\mathcal{G} = (\mathcal{V}, \mathcal{E})\), where \(\mathcal{E}\) denotes the set of edges that specify which pairs of agents can exchange information noiselessly in each round. An agent \(v\) is considered a neighbor of agent \(u\), denoted as \(v \in \mathcal{N}(u)\), if and only if \((v, u) \in \mathcal{E}\).
\subsubsection{Validated Learning Protocol} A Validated Learning Protocol \(\Pi\) is  a decentralized protocol that operates over a fixed number of rounds, say \(T\). As depicted in Figure~\ref{fig:problem}, in each round \(t \in [T]\), each agent \(v\) observes a mini-batch of data \(\vecD_v^{(t)}=(D_{v,1}^{(t)}, D_{v,2}^{(t)}, \ldots, D_{v,K}^{(t)})\), where \(D_{v,k}^{(t)} \in \mathcal{D}\) for all \(k \in [K]\). Subsequently, each agent \(v\) exchanges messages with its neighbors based on the mini-batch $\vecD_v^{(t)}$ and accumulated information from previous rounds. At the conclusion of round \(T\), agents independently compute their \emph{validation state} $S_v$ to be either {``\valid''} or {``\invalid''}. If an agent $v$ computes its final {validation state} to be \valid, then it further outputs its final model vector $\hat{\vecx}_v$. We describe the desired outcome of a validated learning protocol in Section~\ref{sec:validatedlearninggoal}. 
\subsubsection{Adversaries} An unknown subset \(\mathcal{B}\) of agents, termed \emph{Byzantine agents},  may exhibit adversarial behavior. We assume that Byzantine agents have a \emph{causal} global knowledge, \ie, at the outset of the $t$-th round, such agents know all other agents' sent transmissions and observed local data upto (and including) the $t$-th round. Additionally, Byzantine agents have full knowledge of the graph $(\cV,\cE)$, the global data distribution $P_{\cV}$, and the protocol $\Pi$. Armed with this knowledge, each Byzantine agent may produce arbitrary outputs (\ie, they may deviate arbitrarily from $\Pi$) when exchanging messages with their neighbors. Agents in the subset $\cH\triangleq \cV\setminus\cB$ are referred to as \emph{honest agents}. Honest agents are unaware of the presence or identity of Byzantine agents. 
\vspace{-0.25em}
\subsection{Admissible Consensus Models}\vspace{-0.25em}\label{sec:admissiblemodel}
 While the local distributions are not known \emph{a priori}, we assume that it is common knowledge among the agents that \(P_{\mathcal{V}}\) satisfies bounded heterogeneity. In this paper, the specific constraint defined below is of particular interest.

\begin{defn}[\(\delta\)-Heterogeneous Distributions]
We say that the global data distribution \(P_{\mathcal{V}}\in\left(\cP(\cD)\right)^{|\cV|}\) is \(\delta\)-heterogeneous with respect to the loss function $f$ if
\[
\frac{1}{|\mathcal{V}|} \sum_{v \in \mathcal{V}} \ltwo{\mathbb{E}_{D_v \sim P_v} \left[ \nabla_{\vecx} f(\vecx^*(P_{\mathcal{V}}), D_v) \right]}^2 \leq \delta,
\]
where \(\vecx^*(P_{\mathcal{V}})\) is as specified in \eqref{eq:xstar}. We denote the set of  all \(\delta\)-heterogeneous distributions in \(\left(\cP(\cD)\right)^{|\cV|}\) as \(\mathcal{P}_{\delta}\).
\end{defn}


\begin{remark}\label{rem:benignattack}
In the presence of Byzantine agents and heterogeneity, converging to $\vecx^*(P_{\cV})$ is generally infeasible even if honest agents are able to detect if other agents deviate from honest behavior. To see this, consider a "benign" attack strategy orchestrated by Byzantine agents. Armed with the knowledge of \(P_{\cV}\) and $\delta$, an adversary commandeers a subset \(\cB\) of nodes and prescribes specific data distributions \(Q_v\) for each \(v \in \cB\), while ensuring \(P_{\cV \setminus \cB} \times Q_{\cB} \in \cP_{\delta}\). Each Byzantine agent \(v\) then simulates the actions of an honest agent, drawing data from \(Q_v\) instead of \(P_v\). This renders the situation indistinguishable from a case where all agents, including those in \(\cB\), are honest and the actual global data distribution is \(P_{\cV \setminus \cB} \times Q_{\cB}\). Note that $f^*(P_{\cH}\times Q_{\cB})$ is an upper bound on the true global loss value $f^*(P_{\cV})$.  \end{remark}\vspace{-0.25em}

Remark~\ref{rem:benignattack} underscores the need for including adversarial choices in attainable consensus models. We introduce the notion of \emph{admissible consensus models} that capture this notion.\vspace{-0.25em}
\begin{defn}[\((\cU, P_{\cV}, \delta)\)-Admissible Consensus Models]\label{def:admissible}
Given a set of agents \(\cU\), a global data distribution $P_{\cV}$, and a $\delta>0$, a model vector \(\hat{\vecx}\in\bbR^{d}\) is considered  an \((\cU, P_{\cV}, \delta)\)-admissible consensus model if there exist vectors \((\hat{\vecg}_v:v\in\cV\setminus\cU)\)\ such that:\\
\noindent{ a)} $\sum_{u \in \cU} \mathbb{E}_{D \sim P_u} \nabla_{\vecx}f(\hat{\vecx}, D) + \sum_{v \in \cV\setminus\cU} \hat{\vecg}_v=  0$ and \\
\noindent{ b)}$\frac{1}{|\cV|}\left[\sum_{u \in \cU} \ltwo{\mathbb{E}_{D \sim P_u} \nabla_{\vecx}f(\hat{\vecx}, D)}^2 + \sum_{v \in \cV\setminus\cU} \ltwo{\hat{\vecg}_v}^2\right]\leq \delta$.\vspace{0.25em}
We denote  the set of all \((\cU,P_{\cV},\delta)\)-admissible models by $\cA(\cU,P_{\cV},\delta)$.\footnote{Note that when $f$ is convex, for any $P_{\cV}\in\cP_{\delta}$,  $\cA(\cV,P_{\cV},\delta)=\{\vecx^*(P_\cV)\}$. This follows from the uniqueness of the global empirical loss minimizer for convex loss functions.}
\end{defn}\vspace{-0.5em}
\vspace{-0.5em}
\subsection{Protocol objectives}\vspace{-0.25em}
\label{sec:validatedlearninggoal}
The goal of a validated learning protocol is that each honest agent either converges to an admissible consensus model or correctly identifies that there is at least one Byzantine agent in the network. Formally, let $\cH_{\valid}$ (resp. $\cH_{\invalid}$) be the subset of honest agents $v$ that set their validation states $S_v$ to $\valid$ (resp. $\invalid$) when the protocol terminates. Let $\epsilon$ (typically $o(1)$) be a non-negative tolerance parameter. We say a validated learning protocol $\Pi$ terminates $\epsilon$-\emph{successfully} if one of the following outcomes is reached. 
\begin{itemize}
	\item[\emph{a)}] $\cB=\phi$, $\cH_{\valid}=\cV$, and $\frac{1}{|\cV|}\sum_{v\in\cV}
\ltwo{{\vecx}^*(P_{\cV})-\hat{\vecx}_v}^2<\epsilon$.
\item[\emph{b)}] $\cB\neq\phi$,  $\cH_{\valid}\neq\phi$, and there exists $\hat{\vecx}\in\cA(\cH_{\valid},P_{\cV},\delta)$ such that $
\frac{1}{|\cH_{\valid}|}\sum_{v\in\cH_{\valid}}\ltwo{\hat{\vecx}-\hat{\vecx}_v}^2<\epsilon$.
\item[\emph{c)}] $\cB\neq\phi$ and $\cH_{\invalid}=\cH$.
\end{itemize}
\vspace{-0.5em}
\subsection{Assumptions}\vspace{-0.25em} Throughout this paper, assume that the Byzantine set $\cB$ satisfies Assumption~\ref{asm:source}, the loss function $f$ satisfies Assumptions~\ref{asm:bounded},~\ref{asm:smooth} and~\ref{asm:strongcvx}, the local data distributions $P_v$ satisfy Assumptions~\ref{asm:boundedlossvariance} and~\ref{asm:boundedgradientvariance} with respect to $f$, and the global data distribution $P_{\cV}$ satisfies Assumption~\ref{asm:deltahet} with respect to $f$.\vspace{-0.25em}
\begin{assumption}[Existence of a source component]\label{asm:source}
The graph \(\mathcal{G}_{\mathcal{B}^c} \triangleq (\mathcal{V} \setminus \mathcal{B}, \mathcal{E} \setminus \mathcal{E}_{\mathcal{B}})\) is connected, where \(\mathcal{E}_{\mathcal{B}} = \mathcal{E} \cap (\mathcal{B} \times \mathcal{V} \cup \mathcal{V} \times \mathcal{B})\) is the set of edges that are incident on nodes in \(\mathcal{B}\).
\end{assumption}\vspace{-0.5em}
\begin{assumption}[Finite loss]
\label{asm:bounded}
$\sup_{D \in \mathcal{D}} f(\vecx,D)<\infty\ \forall\  \vecx\in\mathbb{R}^d$.
\end{assumption}\vspace{-0.5em}
\begin{assumption}[\(\beta\)-smooth]
\label{asm:smooth}
There exists \(\beta > 0\) such that for all \(\vecx, \vecx' \in \mathbb{R}^d\) and \(D \in \mathcal{D}\),
\[
\ltwo{\nabla_{\vecx} f(\vecx, D) - \nabla_{\vecx} f(\vecx', D)} \leq \beta \ltwo{ \vecx - \vecx' }.\]
\end{assumption}\vspace{-0.75em}
	
\begin{assumption}[\(\mu\)-strongly convex]
\label{asm:strongcvx}
There exists \(\mu > 0\) such that for all \(\vecx, \vecx' \in \mathbb{R}^d\) and \(D \in \mathcal{D}\),
\[
f(\vecx', D) \geq f(\vecx, D) + \langle \nabla_{\vecx} f(\vecx, D), \vecx' - \vecx \rangle + \frac{\mu}{2} \ltwo{\vecx' - \vecx }^2.
\]
\end{assumption}\vspace{-0.75em}


\begin{assumption}[Bounded loss variance] \label{asm:boundedlossvariance} There exists $\sigma_f>0$ such that, for each agent $v$ and for every $\vecx\in\bbR^d$,
\begin{equation}
	\bbE_{D\sim P_v}(f(x,D))^2-(\bbE_{D\sim P_v}f(x,D))^2\leq \sigma_f^2
\end{equation}	
\end{assumption}\vspace{-0.5em}
\begin{assumption}[Bounded gradient variance] \label{asm:boundedgradientvariance} There exists $\sigma_g>0$ such that, for each agent $v$ and for every $\vecx\in\bbR^d$,
\begin{equation}
	\bbE_{D\sim P_v}\ltwo{\nabla_{\vecx}f(x,D)}^2-\ltwo{\bbE_{D\sim P_v}\nabla_{\vecx}f(x,D)}^2\leq \sigma_g^2
\end{equation}	
\end{assumption}\vspace{-0.5em}

\begin{assumption}[\(\delta\)-heterogeneous distribution] There exists $\delta>0$ such that $P_{\cV}$ is \(\delta\)-Heterogeneous with respect to $f$.	\label{asm:deltahet}
\end{assumption}\vspace{-0.5em}
The following assumption is used in Theorem~\ref{thm:optimality}
\begin{assumption}[{$f$-completeness}] \label{asm:consistency} We say that $\cP_{\delta}$ is {$f$-complete}  if for every $P_{\cV}\in\cP_{\delta}$, $\hat{\vecx}\in\bbR^d$, and $(\hat{\vecg}_v:v\in\cB)\in\bbR^{d\times|\cB|}$ satisfying
\begin{align*}
	&\sum_{u \in \cH} \mathbb{E}_{D \sim P_u} \nabla_{\vecx}f(\hat{\vecx}, D) + \sum_{v \in \cB} \hat{\vecg}_v=  0\ \textrm{and}\\
&\frac{1}{|\cV|}\left[\sum_{u \in \cH} \ltwo{\mathbb{E}_{D \sim P_u} \nabla_{\vecx}f(\hat{\vecx}, D)}^2 + \sum_{v \in \cB} \ltwo{\hat{\vecg}_v}^2\right]\leq \delta,
\end{align*}
 there exist distributions $(Q_v:v\in\cB)$ such that $P_{\cH}\times Q_{\cB}\in\cP_{\delta}$ and  $\mathbb{E}_{D \sim Q_v} \nabla_{\vecx}f(\hat{\vecx}, D)=\hat{\vecg}_v$ for each $v\in\cB$. 
	
\end{assumption}\vspace{-0.5em}
\section{Main Results}\label{sec:mainresults}
\vspace{-0.25em}
\begin{theorem}[The \protocol protocol]\label{thm:protocol} Suppose that Assumptions~\ref{asm:source}-\ref{asm:deltahet} are satisfied. Then, there exists a validated learning protocol \protocol over $T$ rounds with the following guarantees under any Byzantine attack, 
\begin{enumerate}
	\item[A.] \textbf{Successful termination:} \protocol terminates $O(1/T)$-successfully with probability $1-\exp(-O(KT))$.
	\item[B.] \textbf{Convergence rate:} When $\cH_{\valid}\neq\phi$, there exists $\hat{\vecx}\in\cA(\cH_{\valid},P_{\cV},\delta)$ such that the final model vectors $(\hat{\vecx}_v:v\in\cH_{\valid})$ satisfy $\sum_{v\in\cH_{\valid}} \ltwo{\hat{\vecx}_v-\hat{\vecx}}^2=O(|\cE|\delta/T)$.
	\item[C.] \textbf{Computational and Communication complexity:} With respect to the model dimension $d$, the number of iterations $T$, and the mini-batch size $K$, for each honest agent, the complexities of \protocol scale as:
	\begin{itemize}
		\item[i)] Computational complexity: $O(c(d)TK+|\cV||\cE|dT)$, where $c(d)$ is the scaling (w.r.t. $d$) of the complexity of evaluating $\nabla_{\vecx} f(\cdot,\cdot)$.
		\item[ii)] Communication complexity: $O(dT+|\cE|^2)$.
	\end{itemize} \end{enumerate}
\end{theorem}
\vspace{-0.5em}
\begin{remark} Even though honest agents do not know {\em a priori} whether or not $\cB=\phi$, the rate of convergence of the model iterates $\vecx_{v}^{(t)}$ to the global optimal consensus model in the setting with $\cB=\phi$ is identical to that of non-Byzantine stochastic gradient descent (where it is known that $\cB=\phi$). Further, regardless of $\cB$, the computational and communication complexities of \protocol scale comparably to those for non-Byzantine distributed stochastic gradient descent~\cite{NedicO:09}.
	
\end{remark}\vspace{-0.5em}
\begin{theorem}[Optimality of~\protocol]\label{thm:optimality} Suppose $(f,\cP_{\delta})$ satisfy Assumptions~\ref{asm:smooth}-~\ref{asm:consistency}. Let $\Pi$ be any validated learning protocol such that under no Byzantine attack, $\Pi$ terminates $\epsilon$-successfully with probability $1-\gamma$. For any $\cB\subsetneq\cV$, $P_{\cV}\in\cP_{\delta}$, and $\hat{\vecx}\in\cA(\cV\setminus\cB,P_{\cV},\delta)$, there exists a Byzantine attack such that, with probability at least $1-\gamma-o(1)$, $\Pi$ terminates {$O(1/T)$-successfully} with $S_v=\valid$ for each $v\in\cV$ and {$\sum_{v\in\cV\setminus\cB}\ltwo{\hat{\vecx}_v-\hat{\vecx}}^2=O(|\cV\setminus\cB|\delta/T)$}. 
\end{theorem}
\section{The \protocol Protocol}\label{sec:validprotocol}
In this section, we describe our proposed \protocol protocol. Due to space constraints, we sketch the guarantees of Theorem~\ref{thm:protocol} in the appendix. \protocol\ (see Protocol~\ref{alg:valid} in the appendix), consists of two phases, the learning phase consisting of the \textsc{LearnModel} sub-protocol (Protocol~\ref{alg:dgd} in the appendix) and the validation phase consisting of the \textsc{ValidateModel} sub-protocol (Protocol~\ref{alg:validationphase} in the appendix). \vspace{-0.5em}
\subsection{Learning phase}\vspace{-0.25em}\label{sec:learningphase}
The learning phase (refer to Protocol~\ref{alg:dgd} in the appendix and Figure~\ref{fig:roundt}), is based on distributed stochastic gradient for non-Byzantine networks (\emph{c.f.}~\cite{NedicO:09}). In each round $t=1,2,\ldots,T$, each agent $v$ computes its model iterate $\vecx_{v}^{(t)}$ and the gradient iterate $\vecg_{v}^{(t)}$ and transmits these to all its neighbors.
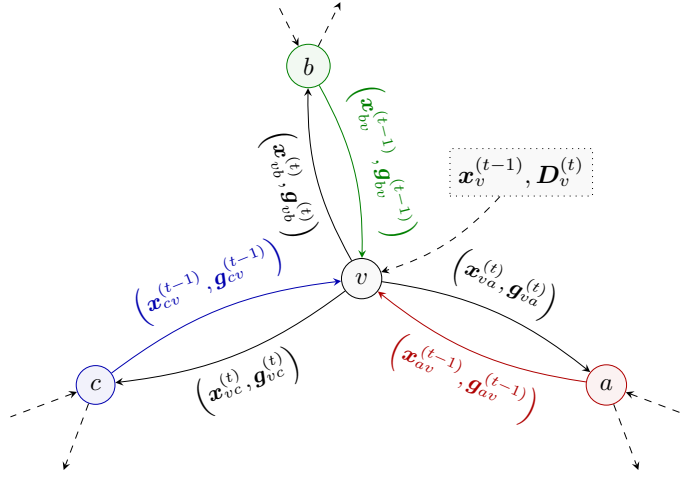
\begin{figure} 
\centering
 
        \resizebox{.5\textwidth}{!}{\begin{tikzpicture}[scale=1.5,>=stealth]
            \node[circle, draw, fill=gray!5!white, minimum size=0.5cm] (v) at (0,0) {$v$};
            \node[circle, draw=darkred, fill=darkred!5!white, minimum size=0.5cm] (a) at (2.3,-1) {$a$};
            \node[circle, draw=darkgreen, fill=darkgreen!5!white, minimum size=0.5cm] (b) at (-0.5,2) {$b$};
            \node[circle, draw=darkblue, fill=darkblue!5!white, minimum size=0.5cm] (c) at (-2.5,-1) {$c$};
           
 			\node[rectangle, draw=black, dotted, fill=gray!5!white, minimum size=0.5cm] (vin) at (1.5,1) {$\vecx_{v}^{(t-1)},\mD_v^{(t)}$};
           
            \draw[->, dashed] (a) -- ++(0.3,-0.8);
            \draw[->, dashed] (b) -- ++(0.3,0.6);
            \draw[->, dashed] (c) -- ++(-0.3,-0.8);
    
            \draw[<-, dashed] (a) -- ++(0.8,-0.3);
            \draw[<-, dashed] (b) -- ++(-0.3,0.6);
            \draw[<-, dashed] (c) -- ++(-0.8,-0.3);
    
            
            \draw[->, black,dashed] (vin) to[bend left=15] node[midway, above, sloped, font=\small] {} (v);
            
            \draw[->, black] (v) to[bend left=15] node[midway, above, sloped, font=\small] {$\left(\vecx_{va}^{(t)},\vecg_{va}^{(t)}\right)$} (a);
            
            \draw[->, darkred] (a) to[bend left=15] node[midway, below, sloped, font=\small] {$\left(\vecx_{av}^{(t-1)},\vecg_{av}^{(t-1)}\right)$} (v);
            
            \draw[->, black] (v) to[bend left=15] node[midway, below, sloped, font=\small] {$\left(\vecx_{vb}^{(t)},\vecg_{vb}^{(t)}\right)$} (b);
            \draw[->, darkgreen] (b) to[bend left=15] node[midway, above, sloped, font=\small] {$\left(\vecx_{bv}^{(t-1)},\vecg_{bv}^{(t-1)}\right)$} (v);

            \draw[->, black] (v) to[bend left=15] node[midway, below, sloped, font=\small] {$\left(\vecx_{vc}^{(t)},\vecg_{vc}^{(t)}\right)$} (c);
            \draw[->, darkblue] (c) to[bend left=15] node[midway, above, sloped, font=\small] {$\left(\vecx_{cv}^{(t-1)},\vecg_{cv}^{(t-1)}\right)$} (v);

        \end{tikzpicture}}
                \caption{The above figure shows the messages passed to node $v$ by its neighbors in round $t-1$  and the messages passed from node $v$ to its neighbor in round $t$ of the \textsc{LearnModel} protocol. Our \textsc{LocalValidation} protocol verifies if these messages are consistent with each other. For instance, if agent $v$ performs all computations honestly, the messages for a single round must satisfy the property that $\vecx_{va}^{(t)}= \vecx_{vb}^{(t)}=\vecx_{vc}^{(t)}$ and $\vecg_{va}^{(t)}= \vecg_{vb}^{(t)}=\vecg_{vc}^{(t)}$ for all $t$. Further, the messages passed in different rounds are related by the equality $\vecx_{va}^{(t)}=(1-3\eta^{(t)}) \vecx_{va}^{(t-1)}+\eta^{(t)}(\vecx_{av}^{(t-1)}+\vecx_{bv}^{(t-1)}+\vecx_{cv}^{(t-1)})-\vecg_{va}^{(t)}$.}\label{fig:roundt}
\end{figure}\vspace{-0.25em}
\subsection{Validation phase}\vspace{-0.25em}\label{sec:validationphase}
The validation phase comprises of local validation, global validation, and validation state agreement sub-protocols. Local validation examines message consistency in relation to the inputs, as shown in Figure~\ref{fig:roundt}. Global validation confirms the agents' model and gradient iterates are consistent with a $\delta$-heterogeneous data distribution. Roughly speaking, successful local and global validations indicate that any Byzantine agents' iterates align with a $(\cV\setminus\cB,P_{\cV},\delta)$-admissible model. If either of these checks fails, the detecting agent's alarm state $S_v^{(0)}$ flips to $\invalid$. The validation state agreement sub-phase ensures unanimous $\invalid$ final states among honest agents if even one alarm state is $\invalid$. However, even when the alarm state of all agents is $\valid$, Byzantine activity during the validation state agreement phase may still lead some honest agents to set  $\invalid$. In this scenario, agents in $\cH_{\valid}$  converge to a valid consensus model, while those in $\cH_{\invalid}$ finalize their states as $\invalid$.

\subsubsection{Primitives}
We introduce the validated broadcast primitive  and the hash function employed in our protocol. 
\paragraph{Validated Broadcast}
This sub-protocol aims to uniformly broadcast a source agent's message, ensuring that honest agents either receive it intact or flag their validation state as $\invalid$. Initially, only the source agent has the target message; others have a null message `$\varnothing$'. Agents exchange messages with neighbors, updating their own only if it was null or flagging $\invalid$ upon discrepancies. After $\min\{|\cV|,|\cE|\}$ exchanges, all honest agents either possess the unaltered message or have flagged $\invalid$. Refer to Protocol~\ref{alg:validatedbroadcast} in the appendix.
\paragraph{Polynomial Hashing} Let $\bbF$ be a large enough finite field. Our hashing schemes are based on the following polynomial hash.  For any vector $\xi\in\bbR^{d(T-2)}$ and $s\in\bbF$, define the hash $\textsc{Hash}:\bbF\times \bbR^{d(T-2)}\to \bbR$ as $\textsc{Hash}(s,\xi)=\sum_{i=1}^{d(T-2)} \xi_i \mbox{int}(s^{i-1})$, where, $\mbox{int}(s^{i-1})$ is integer representation of the finite field element $s^{i-1}$. Note that for a fixed key $s$, $\textsc{Hash}(s,\xi)$ is a linear function of $\xi$.
\subsubsection{Local Validation} \label{sec:localvalidation}

 If agent $v$ is honest, the transcripts on the edges incoming and outgoing from it must satisfy consistency and boundedness properties. Specifically, let $(\vecx_{uv}^{(t)},\vecg_{uv}^{(t)})$ be the pair of vectors transmitted by agent $u$ to agent $v$ at the end of round $t$ of the learning phase. By Lemma~\ref{lem:bounded}, an honest agent's transmissions must be bounded in magnitude. Next, let $\vecX^{\mathrm{out}}_{uv} \triangleq [{\vecx_{uv}^{(t)}}^\top:t=1,2,\ldots,T]^\top$, $\vecX^{\mathrm{in}}_{uv} \triangleq [{\vecx_{uv}^{(t)}}^\top:t=0,1,\ldots,T-1]^\top$, $\vecX^{\mathrm{in},\eta}_{uv} \triangleq [{\eta^{(t)}\vecx_{uv}^{(t)}}^\top:t=0,1,\ldots,T-1]^\top$, and  and $\vecGamma_{uv} =[{\alpha^{(t)}\vecg_{uv}^{(t)}}^\top:t=1,2,\ldots,T]^\top$ denote $\bbR^{dT}$-valued partial transcripts on the edge $(u,v)$. An honest agent $v$ implies that there exist $\overline{\vecX}_v^{\mathrm{out}}, \overline{\vecX}_v^{\mathrm{in}}, \overline{\vecX}_v^{\mathrm{in},\eta}$, and $\overline{\vecGamma}_v $ such that $\overline{\vecX}_v^{\mathrm{out}}=\vecX_{vu}^{\mathrm{out}}$, $\overline{\vecX}_v^{\mathrm{in}}=\vecX_{vu}^{\mathrm{in}}$, $\overline{\vecX}_v^{\mathrm{in},\eta}=\vecX_{vu}^{\mathrm{in},\eta}$, and $\overline{\vecGamma}_v=\vecGamma_{vu}$ for $u\in\cN(v)$. Further, if every agent performs communication and computations with these variables honestly, these variables further satisfy  $\overline{\vecX}_v^{\mathrm{out}}=\overline{\vecX}_v^{\mathrm{in}}+\sum_{u\in\cN(v)}(\overline{\vecX}_u^{\mathrm{in},\eta}-\overline{\vecX}_v^{\mathrm{in},\eta})-\overline{\vecGamma}_v$. See Figure~\ref{fig:allrounds}, in the appendix for a visual depiction.

The \textsc{LocalValidate} protocol first checks whether the norms of the messages transmitted by each neighbouring agent satisfy the bounds of Lemma~\ref{lem:bounded} and then efficiently confirms transcript vector consistency for each agent $v$ by using hash values instead of transmitting entire transcripts, thereby reducing communication cost. Agents first draw a private key $s_v$ to compute and broadcast hash values of their incident transcript vectors. Following this, private keys are broadcast. This sequence prevents Byzantine agents from knowing other agents' keys before sharing their own hash values. Agents then recompute and broadcast hash values of their incident transcript vectors using all received keys. Broadcasts are executed using the validated broadcast protocol. Any discrepancies in consistency checks (or during the validated broadcast) prompt agents to set their alarm state $S_v^{(0)}$ to $\invalid$. Please refer to the appendix for the detailed description (Protocol~\ref{alg:localvalidation} in the appendix).

\subsubsection{Global Validation}\label{sec:globalvalidation}
The intuition behind the global validation phase is that when $\cB=\phi$, for some collection of vectors $(\overline{\vecg}_v^*:v\in\cV)$ taken from $\bbR^d$,  the gradient vectors must satisfy $\ltwo{\bbE_{\mD^{(t)}}\vecg_{vu}^{(t)}-\overline{\vecg}_v^*}^2=O(1/t)$ for all $(v,u)\in\cE$, $ \sum_{v\in\cV} \overline{\vecg}_v^* = 0$, and $\sum_{v\in\cV} \ltwo{\overline{\vecg}_v^*}^2 \leq \delta$.  


\paragraph{Final gradient estimation and broadcast} Let $\gamma$ be small enough (as specified in Lemma~\ref{lem:perturbation} in the appendix). For each neighbor $u\in\cN(v)$, node $v\in\cV$ maintains estimates $\hat{\vecg}_{uv}$ of $\overline{\vecg}_u^*$, $\hat{\ell}_{uv}$ of $\ltwo{\overline{\vecg}_u^*}$ using the following weighted sums:
\begin{align}
\hat{\vecg}_{uv}&  = \frac{1-\gamma^{T-1}}{1-\gamma}\sum_{i=1}^{T-1}\gamma^{T-i-1}\vecg_{uv}^{(i)},\ \mathrm{ and}\\
\hat{\ell}_{uv}&  = \frac{1-\gamma^{T-1}}{1-\gamma}\sum_{i=1}^{T-1}\gamma^{T-i-1}\ltwo{\vecg_{uv}^{(i)}}.
\end{align}

Next, every node $v$ broadcasts $((\hat{\vecg}_{uv},\hat{\ell}_{uv}):u\in\cN(v))$ using the validated broadcast algorithm.

\paragraph{Heterogeneity and optimality test} \label{sec:heterogeneity_optimality} Set a slack parameter $\epsilon>0$. Upon receiving the estimates $((\hat{\vecg}_{uv},\hat{\ell}_{uv}):(u,v)\in\cE)$, the following verification is performed by every agent. 
\begin{enumerate}
 \item {\bfseries Consistency check:} If $(\hat{\vecg}_{uv},\hat{\ell}_{uv})\neq (\hat{\vecg}_{uv'},\hat{\ell}_{uv'})$ for some $v,v'\in\cN(u)$, declare the presence of an adversary. Else, set $(\hat{\vecg}_{u}^*,\hat{\ell}_{u}^*)= (\hat{\vecg}_{uv},\hat{\ell}_{uv})$ for any $v\in\cN(u)$.
 \item {\bfseries Optimality check:} If $\ltwo{\frac{1}{|\cV|}\sum_{u\in\cV} \hat{\vecg}_{u}^*}>\epsilon$, declare the presence of an adversary.
 \item {\bfseries Heterogeneity check:} If $\frac{1}{|\cV|}{\sum_{u\in\cV} \hat{\ell}_{u}^*}^2>\delta+\epsilon$, declare the presence of an adversary.
\end{enumerate}
\vspace{-0.25em}
\subsubsection{The Validation State Agreement sub-phase} The goal of this sub-protocol is to ensure that if any honest agent detects a Byzantine presence, all others recognize this alert. Agents exchange their validation states with neighbors, updating to $\invalid$ if any neighboring state is $\invalid$. After a set number of exchanges (up to $\min\{|\cV|,|\cE|\}$ times), agents concur on $\valid$ if $\cB=\phi$, or on $\invalid$ if Byzantine presence was detected before this sub-protocol.\footnote{Note that even without prior Byzantine detection, some honest agents (\ie, those in $\cH_{\invalid}$) might adopt the validation state $\invalid$ if Byzantine agents start intervening during this sub-protocol. In such cases, all honest agents may not reach consensus on the validation state since the alert raised by an honest agent detecting this Byzantine activity may not propagate to all agents before this sub-phase terminates (in a fixed number of time steps).} See Algorithm~\ref{alg:validateconsensus}  in the appendix for specifics.

\vspace{-0.25em}
\section{Experiments}\label{sec:experiments}\vspace{-0.25em}
We test the performance of \protocol and compare it with two state-of-the-art Byzantine-resilient algorithms for decentralized training,
UBAR \cite{GuoZYXMXL:22} and Bridge-median (Bridge-M) \cite{FangYB:22}. UBAR employs a combination of distance- and performance-based algorithms to mitigate the effect of adversaries. It checks the models sent by neighboring agents and discards outliers (based on the distances between vectors of parameters). {Next, it tests the remaining models against the validation dataset, and takes the average of those with the best performances. Bridge-M screens the incoming models and applies a coordinate-wise median on the received parameters.
\subsubsection{Network} We consider an undirected graph with $n=20$ agents. The graph consists of two fully connected subgraphs of 10 agents each. These two subgraphs are connected with two edges chosen randomly. This ensures that there is a benign path between any two benign agents, provided there is at most one adversary. The graph model is presented in Appendix \ref{graph_model}.
\subsubsection{Dataset and hyperparameters} We evaluate our validation scheme by training a Neural Network (NN) with three fully connected layers. Each of the first two fully connected layers is followed by ReLU, while softmax is used at the output of the third layer.
We used a batch size of 300 and a decaying learning rate of $\frac{lr_{init}}{lr_{init}+T}$, where $lr_{init}=5$ denotes the initial learning rate.

We consider two settings of data distributions - i.i.d. and moderately non-i.i.d., both based  on the MNIST dataset \cite{deng2012mnist}. To simulate independent and identically distributed (i.i.d.) data of every agent, the whole training dataset is shuffled randomly and partitioned equally among the agents.  To model moderately non-i.i.d data of the agents, we first distribute i.i.d. data, as described above. 
Next, each of the 10 agents in the first subgraph randomly and independently picks four label classes out of ``0'', ``1'', ``2'', ``3'', and ``4'', and rotates the images from these four classes by 90 degrees. The procedure is repeated for the 10 agents in the second subgraph, but now the four label classes are chosen from ``5'', ``6'', ``7'', ``8'', and ``9''.
As a result, the data distributions among the agents within each subgraph are closer to each other than to the data distribution of the agents in the other subgraph. Also, each image in the test dataset is rotated with a probability 0.4. 


We now present our results for the moderately non-i.i.d. setting. Due to limited space, the results for i.i.d data distribution on MNIST are presented in Appendix~\ref{iid_result}.

\vspace{-0.25em}
\subsection{No-adversary case}\vspace{-0.25em}


\begin{figure}
    \centering
    \includegraphics[width=0.7\linewidth]{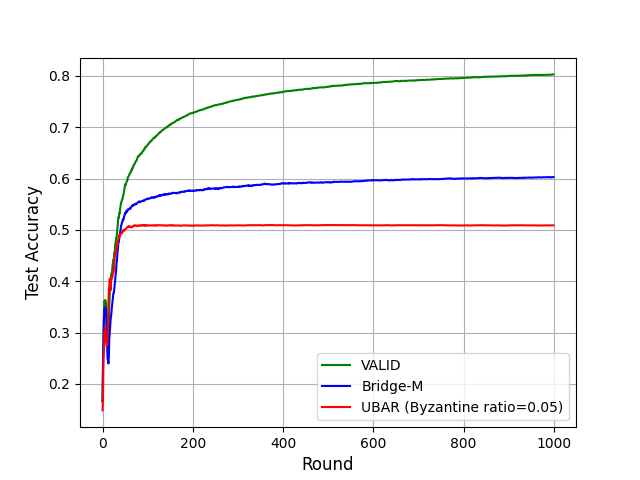}
    \caption{Performance of \protocol compared with UBAR and Bridge-M under non-i.i.d. setting.}
    \label{fig:noniid}
\end{figure}
Fig.\ref{fig:noniid} compares the performances of UBAR\footnote{Note that UBAR scheme is not {specifically} designed for non-i.i.d. setting.}, Bridge-M, and \protocol schemes in non-i.i.d. setting. 
The negative effect of non-i.i.d. data is more pronounced for UBAR and Bridge-M schemes. As UBAR is a performance-based scheme, it is hard for the agents to
distinguish whether a high loss value for a given model sent by a neighboring agent can be attributed to a potential attack or to the heterogeneous nature of the data.
Performance of the Bridge-M scheme also degrades in non-i.i.d setting, due to the coordinate-wise median choice of the parameters. 
To the contrary, in \protocol, agents take the weighted average of all models received from their neighboring agents.
\vspace{-0.25em}
\subsection{Effect of an attack}\vspace{-0.25em}
We consider the following attack model:
an adversary adds Gaussian noise $\mathcal N(0,\sigma^2)$ to each of its model parameters before sending them to its neighboring agents.
\begin{figure}
    \centering
    \includegraphics[width=0.7\linewidth]{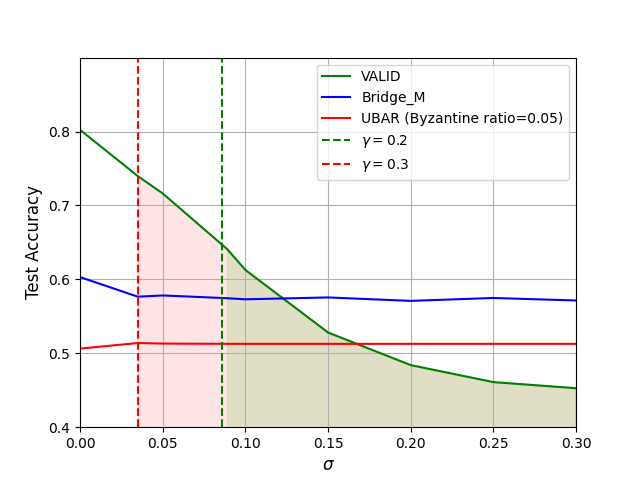}
    \caption{Performance of \protocol, UBAR, and Bridge-M under Gaussian attack (1 adversarial agent, non-i.i.d. setting).}
    \label{fig:noniid_attack}
\end{figure}
While \protocol demonstrates good performance in non-i.i.d setting, it is not a Byzantine-resilient scheme and its performance degrades in the presence of an adversary as the variance of the noise $\sigma$ increases. In contrast, UBAR and Bridge-M are Byzantine-resilient schemes and their performances remain almost the same in the presence of an adversary with different values of $\sigma$ as depicted in Fig. \ref{fig:noniid_attack}. However, \protocol is able to detect the attack using the heterogeneity and optimality checks presented in section \ref{sec:globalvalidation}. As depicted in Fig. \ref{fig:noniid_attack}, the sensitivity of \protocol to detect the attack increases with $\gamma$. Therefore, it can detect weak attacks with the appropriate choice of $\gamma$.
We experimentally select $\gamma\leq\gamma_{max}$, where $\gamma_{\text{max}}$ is the largest value that can successfully pass the heterogeneity check without triggering a false alarm in the non-i.i.d. setting with no attack.

\bibliographystyle{IEEEtran}
\bibliography{Decentralized_Learning}
\clearpage\pagebreak\clearpage
\pagebreak
\clearpage
\appendix
\subsection{Algorithmic description of\protocol}
\begin{algorithm}[h]
    \SetKwInOut{KwIn}{Input}
    \SetKwInOut{KwOut}{Output}
  \KwIn{Number of rounds $T$\\ Step sizes $\{\alpha^{(1:T)},\eta^{(1:T)}\}$\\ Data $\tD=(\mD_{v}^{(1:T)}:v\in\cV)$}
  \KwOut{Final models $\hat{\vecx}_\cV$\\ 
  Final validation states $S_\cV$}
 	$({\vecx}_\cE^{(1:T)},{\vecg}_\cE^{(1:T)})\gets$\textsc{LearnModel($T,\{\alpha^{(1:T)},\eta^{(1:T)}\},\tD$)}\\
    $(S_\cV,\hat{\vecx}_\cV)\gets$\textsc{ValidateModel(${\vecx}_\cE^{(1:T)},{\vecg}_\cE^{(1:T)}$)}
\caption{The \protocol Protocol}\label{alg:valid}
\end{algorithm}
\subsubsection{Learning phase}
In the learning phase, agents perform distributed stochastic gradient for non-Byzantine networks (\emph{c.f.}~\cite{NedicO:09}). This phase of the protocol runs over $T-1$ rounds. In each round, the agents compute their model iterates and transmit the model iterate as well as the value of the computed gradient to all their neighbors. Fix a sequence $\{(\alpha^{(t)},\eta^{(t)})\}_{t\in[T-1]}$ that satisfies Theorem~\ref{thm:referencethm}. Let $\vecx_v^{(t)},\ \vecm_{vu}^{(t)}$, and $\vecg_{vu}^{(t)}$ be $d$-dimensional real vectors denoting agent $v$'s model iterate, the model sent by agent $v$ to agent $u$ and the scaled gradient sent by agent agent $v$ to agent $u$ in round $t$. Protocol~\ref{alg:dgd} details the algorithmic steps. 
\begin{algorithm}
    \SetKwInOut{KwIn}{Input}
    \SetKwInOut{KwOut}{Output}
  \KwIn{Number of rounds $T$\\ Step sizes $\{\alpha^{(1:T)},\eta^{(1:T)}\}$\\ Data $\tD=(\mD_{v}^{(1:T)}:v\in\cV)$}
  \KwOut{Messages $(\vecx_\cE^{(1:T)},\vecg_\cE^{(1:T)})$}
 	$\vecx_v^{(0)}\gets 0^d$ for each $v\in\cV$\\
 	\For{$t=0,2,\ldots,T-1$}{
   		\For{$v\in\cV$}{ 
    		$\vecy_v^{(t+1)}\gets \vecx_v^{(t)}+\eta^{(t+1)}\sum_{u\in\cN(v)}(\vecx_u^{(t)}-\vecx_v^{(t)})$\\
            $\vecg_{v}^{(t+1)}\gets \frac{1}{K}\sum_{k=1}^K \nabla_{\vecx}\ f(\vecy_v^{(t+1)},D_{v,k}^{(t+1)})$\\
    		$\vecx_v^{(t+1)} \gets  \vecy_v^{(t+1)}-\alpha^{(t+1)}\vecg_{v}^{(t+1)}$\\
    	}
    	\For{$u\in\cN(v)$}{
    			 $(\vecx_{vu}^{(t+1)},\vecg_{vu}^{(t+1)})\gets (\vecx_{v}^{(t+1)},\vecg_{v}^{(t+1)})$
   		} 
    	
    }
\caption{\textsc{LearnModel}}\label{alg:dgd}
\end{algorithm}

The detailed description of the protocol is given below.
\paragraph{Intialization:}  Set $\vecx_v^{(0)}=\vecm_{uv}^{(0)}=\vecg_{uv}^{(0)}=0^d$ for all $(u,v)\in\cE$. 

\paragraph{Rounds $1,2,\ldots,T-1$:} At the beginning of each round $t=1,2,\ldots,T-1$, each agent $v$ has access to its previous model iterate $\vecx_v^{(t-1)}$, the collection of messages $((\vecm_{uv}^{(t-1)},\vecg_{uv}^{(t-1)}): u\in\cN(v))$ received from all neighboring agents in round $t-1$, and its local data mini-batch $D_{v,1}^{(t)}, D_{v,2}^{(t)},\ldots,D_{v,K}^{(t)}$. In round $t$, agent $v$ first mixes its local model iterate with those from its neighbors to compute 
\begin{equation}
	\vecy_v^{(t)} =\vecx_v^{(t-1)}+\eta^{(t)}\sum_{u\in\cN(v)}(\vecx_u^{(t-1)}-\vecx_v^{(t-1)}).
\end{equation}
Next, agent $v$ produces the updated model iterate $\vecx_v^{(t)}$ by performing an iteration of stochastic gradient descent with $\vecy^{(t)}$ as the starting model iterate as follows
\begin{equation}
	\vecx_v^{(t)} = \vecy_v^{(t)}-\alpha^{(t)} \frac{1}{K}\sum_{k=1}^K \nabla_{\vecx}\ f(\vecy_v^{(t)},D_{v,k}^{(t)}).
\end{equation}
Lastly, agent $v$ transmits the message $(\vecm_{vu}^{(t)},\vecg_{vu}^{(t)})\triangleq (\vecx_v^{(t)},(\vecx_v^{(t)}-\vecy_v^{(t)})/\alpha^{(t)})$ to each neighboring agent $u\in\cN(v)$. Figure~\ref{fig:roundt} shows the variables involved in node $v$'s round $t$ computations.

\paragraph{Termination:} The learning phase terminates after round $T-1$. At this point, the final model estimate $\hat{\vecx}_v$ of each agent $v$ is set to be $\vecx_v^{(T-1)}$.

\subsubsection{Validation phase}
\begin{algorithm}
    \SetKwInOut{KwIn}{Input}
    \SetKwInOut{KwOut}{Output}
  \KwIn{Model iterates $\vecx_\cE^{(1:T)}$\\Gradient iterates $\vecg_\cE^{(1:T)}$}
  \KwOut{Validation states $S_\cV$\\ Final model vectors $\hat{\vecx}_{\cH_{\valid}}$}
 	$S_\cV^{(0)}\gets\valid^{|\cV|}$\\
 	$S_\cV^{(0)}\gets \textsc{LocalValidate($\vecx_\cE^{(1:T)},\vecg_\cE^{(1:T)}$)}$\\
    $S_\cV^{(0)}\gets \textsc{GlobalValidate($\vecx_\cE^{(1:T)},\vecg_\cE^{(1:T)},S_{\cV}^{(0)}$)}$\\
    $S_\cV\gets \textsc{ValidationStateAgreement($S_\cV^{(0)}$)}$\\
    \For{$v$ s.t. $S_v=\valid$}{$\hat{\vecx}_v\gets \vecx^{(T)}_{vu}$ for some $u\in\cN(v)$}
\caption{\textsc{ValidateModel}}\label{alg:validationphase}
\end{algorithm}
The Validation Phase (Protocol~\ref{alg:validationphase}) consists of two sub-phases -- local verification and global validation. The local verification sub-phase checks if each agents outputs are consistent with respect to its inputs. If this check is passed for all agents, we move on to the global validation sub-phase where it is verified if all agents' final models are $(\epsilon,\cH,\cP_{\delta})$-admissible. Before describing these sub-phases, we introduce two primitives that are repeatedly employed in our protocol.
\begin{figure}[t] 
\centering
  \resizebox{.5\textwidth}{!}{\begin{tikzpicture}[scale=1.5,>=stealth]
            \node[circle, draw, fill=gray!5!white, minimum size=0.5cm] (v) at (0,0) {$v$};
            \node[circle, draw=darkred, fill=darkred!5!white, minimum size=0.5cm] (a) at (2.3,-1) {$a$};
            \node[circle, draw=darkgreen, fill=darkgreen!5!white, minimum size=0.5cm] (b) at (-0.5,2) {$b$};
            \node[circle, draw=darkblue, fill=darkblue!5!white, minimum size=0.5cm] (c) at (-2.5,-1) {$c$};
           
 			\node[rectangle, rounded corners, draw=black, dashed, fill=white, minimum size=0.5cm,align=left] (vin) at (1.5,1) {$\vecX^{\mathrm{in}}_{va},\vecX^{\mathrm{in}}_{vb},\vecX^{\mathrm{in}}_{vc}$\\$\vecX^{\mathrm{in},\eta}_{va},\vecX^{\mathrm{in},\eta}_{vb},\vecX^{\mathrm{in},\eta}_{vc}$};
           
            \draw[->, dashed] (a) -- ++(0.3,-0.8);
            \draw[->, dashed] (b) -- ++(0.3,0.6);
            \draw[->, dashed] (c) -- ++(-0.3,-0.8);
    
            \draw[<-, dashed] (a) -- ++(0.8,-0.3);
            \draw[<-, dashed] (b) -- ++(-0.3,0.6);
            \draw[<-, dashed] (c) -- ++(-0.8,-0.3);
      
            
            \draw[->, black,dashed] (vin) to[bend left=15] node[midway, above, sloped, font=\small] {} (v);
            
            \draw[->, black] (v) to[bend left=15] node[midway, above, sloped, font=\small] {$\left(\vecM^{\mathrm{out}}_{va},\vecGamma_{va}\right)$} (a);
            
            \draw[->, darkred] (a) to[bend left=15] node[midway, below, sloped, font=\small] {$\vecX^{\mathrm{in}}_{av}$} (v);
            
            \draw[->, black] (v) to[bend left=15] node[midway, below, sloped, font=\small] {$\left(\vecX^{\mathrm{out}}_{vb},\vecGamma_{vb}\right)$} (b);
            \draw[->, darkgreen] (b) to[bend left=15] node[midway, above, sloped, font=\small] {$\vecX^{\mathrm{in}}_{bv}$} (v);

            \draw[->, black] (v) to[bend left=15] node[midway, below, sloped, font=\small] {$\left(\vecX^{\mathrm{out}}_{vc},\vecGamma_{vc}\right)$} (c);
            \draw[->, darkblue] (c) to[bend left=15] node[midway, above, sloped, font=\small] {$\vecX^{\mathrm{in}}_{cv}$} (v);

        \end{tikzpicture}}
        \caption{The above figure represents the variables  that participate in the computations undertaken by agent $v$ during the entire learning phase (in contrast to Figure~\ref{fig:roundt} that only represents the round-$t$ variables). }\label{fig:allrounds}
\end{figure}
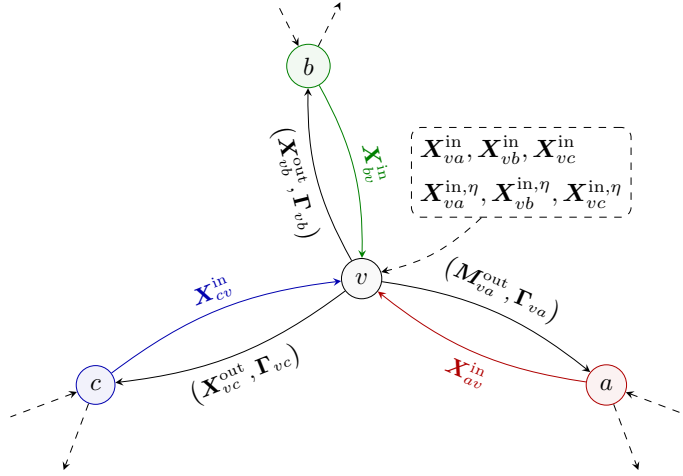

\paragraph{Validated Broadcast} The validated broadcast primitive is described in Protocol~\ref{alg:validatedbroadcast} and forms the fundamental building block for agents to broadcast messages to all agents with the guarantee that any errors in the broadcast are flagged as byzantine in nature.
\begin{algorithm}
\caption{\textsc{ValidatedBroadcast}}\label{alg:validatedbroadcast}
\KwIn {Source agent $v_0$; Message $m$}
\KwOut{Received messages $(m_v:v\in\cV)$, output validation states $(S_v^{(0)}:v\in\cV)$} 
$R\gets\min\{|\cV|,|\cE|\}$\\
$m_{v_0}\gets m$\\
$m_v\gets \varnothing$ for each $v\in\cV\setminus\{v_0\}$\\
\For{$r=1,2,\ldots,R$}
{\For{$v\in\cV$}{
Send $m_v$ to $\cN(v)$}
 \For{$v\in\cV$}{
\For{$u\in\cN(v)$}{
\uIf{$m_v=\varnothing$}{ $m_v\gets m_u$}\uElseIf{$m_u\notin\{\varnothing,m_v\}$}{$S_v^{(0)}\gets\invalid$}}
}
}
\end{algorithm}
\paragraph{Local Validation}

\begin{algorithm}
\caption{\textsc{LocalValidate}}\label{alg:localvalidation}
\KwIn {Transcript vectors $\vecX^{\mathrm{out}}_{uv}$, $\vecX^{\mathrm{in}}_{uv}$, $\vecX^{\mathrm{in},\eta}_{uv}$, $\vecGamma_{uv}$:  $(u,v)\in\cE$}
\KwResult{Validation states $(S_v:v\in\cV)$}
\tcc{Verify that the iterates are bounded} \For {$(u,v)\in\cE$}
{
    \If{$\ltwo{\vecx_{uv}^{(t)}}>B_t$ OR $\ltwo{(\vecg_{uv}^{(t)}}>C_t$} 
    {$S_{v}^{(0)}\gets\invalid$}
}
\For{$v\in\cV$}
{
	Pick $s_v\sim\textrm{Unif}(\bbF)$\\
	\tcc{Compute hashes using own key}
	\For{$u\in\cN(v)$}{
		$h_{u,v}^{v,\vecX_{\mathrm{out}}} \gets \textsc{Hash}(s_v,\vecX^{\mathrm{out}}_{uv})$\\
		$h_{u,v}^{v,\vecX_{\mathrm{in}}} \gets \textsc{Hash}(s_v,\vecX^{\mathrm{in}}_{uv})$\\
		$h_{u,v}^{v,\vecX_{\mathrm{in}},\eta} \gets \textsc{Hash}(s_v,\vecX^{\mathrm{in},\eta}_{uv})$\\
		$h_{u,v}^{v,\vecGamma} \gets \textsc{Hash}(s_v,\vecGamma_{uv})$
	}
}
\tcc{Broadcast hash values}
\For{$v\in\cV$, $u\in\cN(v)$}{
    $m_{u,v}^{(v)}\gets(h_{u,v}^{v,\vecX_{\mathrm{out}}},h_{u,v}^{v,\vecX_{\mathrm{in}}},h_{u,v}^{v,\vecX_{\mathrm{in}},\eta},h_{u,v}^{v,\vecGamma})$
	\textsc{ValidatedBroadcast}$\left(v,m_u^{(v)}\right)$
}
\tcc{Broadcast keys}
\For{$v\in\cV$}{
	\textsc{ValidatedBroadcast($v$,$s_v$)}
}
\tcc{Compute and broadcast hashes using other agents' keys}
\For{$v,v'\in\cV \mbox{ s.t. } v\neq v'$, $u\in\cN(v)$}{
	$h_{u,v}^{v',\vecX_{\mathrm{out}}} \gets \textsc{Hash}(s_{v'},\vecX^{\mathrm{out}}_{uv})$\\
	$h_{u,v}^{v',\vecX_{\mathrm{in}}} \gets \textsc{Hash}(s_{v'},\vecX^{\mathrm{in}}_{uv})$\\
	$\tilde{h}_{u,v}^{v,\vecX_{\mathrm{in}}} \gets \textsc{Hash}(s_{v'},\vecX^{\mathrm{in},\eta}_{uv})$\\	
	$h_{u,v}^{v',\vecGamma} \gets \textsc{Hash}(s_{v'},\vecX^{\mathrm{in}}_{uv})$\\
	$m_{u,v}^{(v')}\gets(h_{u,v}^{v',\vecX_{\mathrm{out}}},h_{u,v}^{v',\vecX_{\mathrm{in}}},h_{u,v}^{v',\vecX_{\mathrm{in}},\eta},h_{u,v}^{v',\vecGamma})$\\
    \textsc{ValidatedBroadcast}$\left(v,m_{u,v}^{(v')}\right)$
}
\For{$v'\in\cV$}{
\tcc{Consistency checks by $v'$}
    \For{$v\in\cV$, $u,u'\in\cN(v)$}
    {
        \uIf{$h_{v,u}^{v',\vecX_{\mathrm{out}}}\neq h_{v,u'}^{v',\vecX_{\mathrm{out}}}$}{$S_{v'}^{(0)}\gets \invalid$}\uElseIf{$h_{v,u}^{v',\vecX_{\mathrm{in}}}\neq h_{v,u'}^{v',\vecX_{\mathrm{in}}}$}{$S_{v'}^{(0)}\gets \invalid$}\uElseIf{$h_{v,u}^{v',\vecX_{\mathrm{in}},\eta}\neq h_{v,u'}^{v',\vecX_{\mathrm{in}},\eta}$}{$S_{v'}^{(0)}\gets \invalid$}\uElseIf{$h_{v,u}^{v',\vecGamma}\neq h_{v,u'}^{,\vecGamma}$}{$S_{v'}^{(0)}\gets \invalid$}\uElse{$(h_{v}^{v',\vecX_{\mathrm{out}}},h_{v}^{v',\vecX_{\mathrm{in}}},h_{v}^{v',\vecX_{\mathrm{in}},\eta},h_{v}^{v',\vecGamma})\gets m_{u,v}^{(v')}$}
    }
    \For{$v\in\cV$, $v''\in\cN(v)$}{
    	\If{$h_{v}^{v',\vecX_{\mathrm{out}}}\neq h_{v}^{v',\vecX_{\mathrm{in}}}+\sum_{v''\in\cN(v)}(h_{v''}^{v',\vecX_{\mathrm{in}},\eta}-h_{v}^{v',\vecX_{\mathrm{in}},\eta})+h_v^{v',\Gamma}$}{$S_{v'}^{(0)}\gets\invalid$}
    }
}
\end{algorithm}
Recall the following transcript vectors. 
\begin{align} 
 	&\vecX^{\mathrm{out}}_{uv} \triangleq [{\vecx_{uv}^{(t)}}^\top:t=1,2,\ldots,T]^\top\label{eq:Mout}\\
 	&\vecX^{\mathrm{in}}_{uv} \triangleq [{\vecx_{uv}^{(t)}}^\top:t=0,1,\ldots,T-1]^\top\label{eq:Min}\\
 	&\vecX^{\mathrm{in},\eta}_{uv} \triangleq [{\eta^{(t)}\vecx_{uv}^{(t)}}^\top:t=0,1,\ldots,T-1]^\top\label{eq:Mineta}\\
    &\vecGamma_{uv} =[{\alpha^{(t)}\vecg_{uv}^{(t)}}^\top:t=1,2,\ldots,T]^\top\label{eq:Gamma}
\end{align}

We note that if an agent $v$ performs all computations honestly, there exist $\overline{\vecX}_v^{\mathrm{out}},\overline{\vecX}_v^{\mathrm{in}},\overline{\vecX}_v^{\mathrm{in},\eta},\overline{\vecGamma}_v\in\bbR^{dT}$ such that the following relations are satisfied by the transcripts received by it and transmitted by it.
\begin{align}
	\vecX_{vu}^{\mathrm{out}}&= \overline{\vecX}_v^{\mathrm{out}} \ \forall\ u\in\cN(v),\\
	\vecX_{vu}^{\mathrm{in}}&= \overline{\vecX}_v^{\mathrm{in}} \ \forall\ u\in\cN(v),\\
	\vecX_{vu}^{\mathrm{in},\eta}&= \overline{\vecX}_v^{\mathrm{in},\eta} \ \forall\ u\in\cN(v),\\
	\vecGamma_{vu}&= \overline{\vecGamma}_v \ \forall\ u\in\cN(v),\mathrm{ and}\\
	\overline{\vecX}_{v}^{\mathrm{out}}&=\overline{\vecX}_{v}^{\mathrm{in}}+\sum_{u\in\cN(v)}(\vecX_{uv}^{\mathrm{in},\eta}-\vecM_{uv}^{\mathrm{in},\eta})-\overline{\vecGamma}_{v}
\end{align}
Thus, it is necessary that for every $s\in\bbF$, the following relations hold for some $H_v^{\mathrm{out}},H_v^{\mathrm{in}},H_v^{\mathrm{in},\eta},H_v^{\vecGamma}\in\bbR$.
	\begin{align}
	\textsc{Hash}(s,\vecX_{vu}^{\mathrm{out}})&= H_v^{\mathrm{out}}\ \forall\ u\in\cN(v),\label{eq:cc1}\\
	\textsc{Hash}(s,\vecX_{vu}^{\mathrm{in}})&=H_v^{\mathrm{in}}\ \forall\ u\in\cN(v),\label{eq:cc2}\\
	\textsc{Hash}(s,\vecX_{vu}^{\mathrm{in},\eta})&=H_v^{\mathrm{in},\eta}\ \forall\ u\in\cN(v),\label{eq:cc2eta}\\
	\textsc{Hash}(s,\vecGamma_{va})&= H_v^{\mathrm{in},\vecGamma}\ \forall\ u\in\cN(v),\ \mathrm{ and}\label{eq:cc3}\\
	H_v^{\mathrm{out}}&=H_v^{\mathrm{in}}+\sum_{u\in\cN(v)}(H_u^{\mathrm{in},\eta}-H_v^{\mathrm{in},\eta})-H_v^{\vecGamma}\label{eq:cc4}
	\end{align}
	Note that the last equality holds since for every $s\in\bbF$, $h(s,\cdot)$ is a linear map. We say that node $v$ has consistent hashes with respect to the key $s$ if~\eqref{eq:cc1}-~\eqref{eq:cc4} are satisfied.

\paragraph{Hash computation and broadcast}
\begin{enumerate}
	\item {\bf Random private key:} First each node $v$  draws its private key $s_v\sim \mathrm{Unif}(\bbF)$ 
	\item {\bf Local hash computation:} 
Every node $v$ performs the following hash computations. Each of the vectors defined in~\eqref{eq:Min}-~\eqref{eq:Gamma} are hashed using the polynomial hash with key $s_v$ to compute the hashes $h_{u,v}^{v,\vecM_{\mathrm{in}}}$, $h_{u,v}^{v,\vecM_{\mathrm{out}}}$, and $h_{u,v}^{v,\vecGamma}$ of  $\vecM^{\mathrm{in}}_{uv}$, $\vecM^{\mathrm{out}}_{uv}$, and $\vecM^{\mathrm{in}}_{uv}$ respectively, for each $u\in\cN(v)$.
\item {\bf Local Hash broadcast}: Each node $v$ broadcasts all the local hashes computed so far to all the nodes in the network using the { validated broadcast sub-protocol}.
\item {\bf Private key broadcast}: After receiving every other node's local hashes, node $v$ broadcasts $s_v$ to all the nodes in the network using the { validated broadcast sub-protocol}.
\item {\bf Cross-hash computation and broadcast}: Each node $v$ now computes a hash of its local vectors defined in ~\eqref{eq:Min}-~\eqref{eq:Gamma} by using the key $s_{v'}$ for each $v'\neq v$ and broadcasts all the results to all the nodes in the network using the { validated broadcast protocol}.
 \end{enumerate}

\paragraph{Local consistency checks:} Each node $v\in\cV$ verifies if each nodes $v'\in\cN(v)$ has consistent hashes with respect to the key $K_v$. If there exists a $v'$ that does not have consistent hashes, then node $v$ declares that there is an adversary and broadcasts this using the validated broadcast protocol.

\subsubsection{Validation State Agreement} See Protocol~\ref{alg:validateconsensus} for details of the algorithm.
\begin{algorithm}
\caption{\textsc{ValidationStateAgreement}}\label{alg:validateconsensus}
\KwIn {Graph $\cG=(\cV,\cE)$; Initial validation states $(S_v^{(0)}:v\in\cV)$}
\KwResult{Final validation states $(S_v:v\in\cV)$}
Set $R=|\cE|$\\
\For{$r=1,2,\ldots,R$}
{\For{$v\in\cV$}{
Send $S_v^{(r-1)}$ to $N(v)$}
 \For{$v\in\cV$}{
\If{$\exists u\in N(v)$ s.t. $S_u^{(r-1)}=\invalid$}{
Set $S_v^{(r)}=\invalid$}
}
}
Set $(S_v:v\in\cV)=(S_v^{(R)}:v\in\cV)$
\end{algorithm}

\subsection{Proof of Theorem~\ref{thm:protocol}}\label{sec:analysis}
In the following, we adopt a matrix notation to compactly represent the collection of vectors across all agents. For example, $\vecX^{(t)}=[\vecx_v^{(t)}:v\in\cV]$ is the $d\times |\cV|$ matrix formed by stacking together the model vectors of all agents as columns. 
\subsubsection{No byzantine agents}
When there are no byzantine nodes in the network at iteration $t$, the model states evolve according to the following equation.
\begin{align}
	&\mX^{(1)} = \mathbf{0}\\
	&\mX^{(t+1)} =\mX^{(t)} \mW(\eta^{(t)}) -\alpha^{(t)} \nabla F(\mX^{(t)}\mW(\eta^{(t)}),\tD^{(t)}). \label{eq:noadvev}
\end{align}

The following statement gives a bound on the rate of convergence of this evolution.
\begin{theorem}[\cite{JakovBSK:18,SahuJBK:18}]\label{thm:referencethm}
Let $\alpha^{(t)}=\alpha_0/(t+1)$ and $\eta^{(t)}=\eta_0/(t+1)^{1/2}$. Then, we have, 
\begin{equation}
	\mathbb{E}\ltwo{\vx_v^{(t)}-\vx^*}=O(\delta/t)+O(1/t).
\end{equation}	
\end{theorem}

Further, when there are no Byzantine agents, all the validation checks are satisfied with probability $1-\exp(O(-KT))$. The error event here corresponds to the some agents' observed data being anomalous. The bound on the error probability follows from standard concentration arguments and is skipped here in the interest of space.

\subsubsection{Byzantine agents present} In this case, we first note that if any validation checks are failed then, the algorithm terminates successfully with honest agents declaring adversarial presence. Thus, we only need to show that when adversarial presence is not detected, the honest agents output an admissible consensus model. To this end, we first show in  Lemma~\ref{lem:validated} that the agents' model parameters are within $O(1/T)$ of a consensus model as long as both local and global validation checks are passed.

We first present a result that establishes a bound on the non-Byzantine iterates as aa function of the given parameters. \begin{lemma}[Bounded iterates]\label{lem:bounded} There exist a sequence $\{B_t,C_t\}_{t=1}^\infty$  of positive finite real numbers depending on depending on $t$, $L$, $\beta$, $\delta$, and $|\cV$ such that when $\cB=\phi$, 
\begin{align}
\ltwo{\vecx_v^{(t)}}\leq B_t,\mbox{ and}
\ltwo{\nabla f(\vecx_v^{(t)},D^{(t)})}\leq C_t
\end{align}
for all $v\in\cV,t=1,2,\ldots,T$.
\end{lemma}
\begin{proof} The proof follows along the lines of~\cite[Lemma~6.1]{SahuJBK:18} by employing the smoothness, strong convexity, and finiteness of the loss function. We skip the proof here.
\end{proof}

\begin{lemma}[Validated models] \label{lem:validated} Let $(\hat{\vecx}_\cE,\hat{\vecg}_\cE)$ satisfy local as well as global validation checks. Then, there exists $\hat{\vecx}^*$ and $(\hat{\vecx}_u:u\in\cV)$ such that we have, with probability $1-O(1/|\bbF|^{|\cH|-1})$, 

\begin{align}
&\hat{\vecx}_{uv}=\hat{\vecx}_u,\mbox{ and}\label{eq:equality}\\
&\sum_{e\in\cE}\ltwo{\hat{\vecx}_{e}-\hat{\vecx}^*}^2=O(1/T)\label{eq:consensus}
\end{align}
\end{lemma}
\begin{proof}
The first condition follows from the local validation checks. Note that the probability (under the random choice of $(s_{v''}:v''\in \cH\setminus \{u\})$) that $\textsc{Hash}(s_{v''},\vecX_{uv}^{out})$ equals $\textsc{Hash}(s_v',\vecX_{uv'}^{out})$ even though $\vecX_{uv}^{out}\neq \vecX_{uv'}^{out}$ for some $v\neq v'$ is at most $1/|\bbF|^{|\cH|-1}$. Thus, as long there is no hash collision,~\eqref{eq:equality} is satisfied. Next, to see that~\eqref{eq:consensus} is satisfied, we first note that since the global validation check is satisfied, the final gradient estimate $\hat{\vecg}_u^*$ satisfies the condition $\ltwo{\hat{\vecg}_u^*}<|\cV|(\delta+\epsilon)$. Further, by Lemma~\ref{lem:bounded}, we know that the gradients are absolutely bounded since the Local Validation is satisfied. Applying Lemma~\ref{lem:perturbation}, we obtain that any bounded Byzantine perturbations decay exponentially, implying that the model parameters converge to a consensus value at the claimed rate.
\end{proof}

Now, suppose that any byzantine nodes are undetected at the beginning of iteration $t$. We model the presence of such failures as that of adding a perturbation $\mZ^{(t)}$ to the model update transmitted by honest nodes. Let $\hat{\mX}^{(t)}$ denote the model when undetected byzantine nodes are present. The model states evolve according to the equation 

\begin{align}
	&\hat{\mX}^{(1)} = \mathbf{0}\\
	& \tilde{\mX}^{(t)} = \hat{\mX}^{(t)}+\mZ^{(t)}\\
	&\hat{\mX}^{(t+1)} =\tilde{\mX}^{(t)} \mW -\alpha^{(t)} \hat{\mG}^{(t)},\label{eq:advev}
\end{align}
where, $\mZ^{(t)}= [\vecz^{(t)}_v:v\in\cV]$ with $\vecz^{(t)}_v = 0$ for $v\notin\cB$ and $\hat{\mG}^{(t)}=\mG(\tilde{\mX}^{(t)}\mW,\tD^{(t)})$.

The following lemma relates the model state evolutions in~\eqref{eq:noadvev} and~\eqref{eq:advev}
\begin{lemma}[Perturbation analysis]\label{lem:perturbation}
	Let $f$ satisfy Assumptions~\ref{asm:smooth} and~\ref{asm:strongcvx}. Let $\alpha^{(t)}\in(0, \max\{1,\mu/\beta^2\})$ for all $t$. Then, there exists $\gamma\in(0,1)$ s.t., for all $t$, we have 
	\begin{align}
 \lVert \hat{\mX}^{(t+1)}-\mX^{(t+1)}\rVert_F
& \leq   \sum_{\tau=0}^t \gamma^{t-\tau+1} \lVert \mZ^{(\tau)}\rVert_F.
 \end{align}
	\end{lemma}
\begin{proof}
From~\eqref{eq:noadvev} and~\eqref{eq:advev}, we have
\begin{align}
	\hat{\mX}^{(t+1)}-\mX^{(t+1)}& = \tilde{\mX}^{(t)} \mW -\alpha^{(t)} \hat{\mG}^{(t)} - \mX^{(t)} W +\alpha^{(t)} \mG^{(t)} \\
	& = (\tilde{\mX}^{(t)}  - \mX^{(t)}) \mW - \alpha^{(t)} (\hat{\mG}^{(t)}-\mG^{(t)}).
\end{align}
	Let $\mY=[\vy_v:v\in\cV]\triangleq \mX^{(t)} \mW $, $\tilde{\mY}=[\tilde{\vy}_v:v\in\cV]\triangleq \tilde{\mX}^{(t)} \mW $, and $\mS=[\vs_v:v\in\cV]\triangleq \tilde{\mY}-\mY$ for compactness. The above equation may be restated as	
	
\begin{align}
&\hat{\vx}_v^{(t+1)}	-	\vx_v^{(t+1)}\\& = \tilde{\vy}_v^{(t)}-\vy_v^{(t)}-\alpha^{(t)}(g(\tilde{\vy}_v^{(t)},\mD^{(t)})-g(\vy_v^{(t)},\mD^{(t)}))
\end{align}
for all $v\in\cV$. Thus, 
\begin{align}
	\lefteqn{\ltwo{\hat{\vx}_v^{(t+1)}-\vx_v^{(t+1)}}^2 }\\
	& = (\hat{\vx}_v^{(t+1)}-\vx_v^{(t+1)})^\top (\hat{\vx}_v^{(t+1)}-\vx_v^{(t+1)})\\
	& = \left(\tilde{\vy}_v^{(t)}-\vy_v^{(t)}-\alpha^{(t)}(g(\tilde{\vy}_v^{(t)},\mD^{(t)})-g(\vy_v^{(t)},\mD^{(t)}))\right)^\top   \\
 & \ \left(\tilde{\vy}_v^{(t)}-\vy_v^{(t)}-\alpha^{(t)}(g(\tilde{\vy}_v^{(t)},\mD^{(t)})-g(\vy_v^{(t)},\mD^{(t)}))\right)\\
	& = \ltwo{\tilde{\vy}_v-\vy_v}^2+ (\alpha^{(t)})^2 \ltwo{g(\tilde{\vy}_v^{(t)},\mD^{(t)})-g(\vy_v^{(t)},\mD^{(t)})}^2\\
 & \ -2\alpha^{(t)}(g(\tilde{\vy}_v^{(t)},\mD^{(t)})-g(\vy_v^{(t)},\mD^{(t)}))^\top(\tilde{\vy}_v-\vy_v).
\end{align}
Next, by Assumption~\ref{asm:smooth}, we have
\begin{align}
	\ltwo{g(\tilde{\vy}_v^{(t)},\mD^{(t)})-g(\vy_v^{(t)},\mD^{(t)})}^2
	& \leq \beta^2 \ltwo{\tilde{\vy_v}^{(t)}-\vy_v^{(t)}}^2.\\
\end{align}
	Further, by Assumption~\ref{asm:strongcvx}, we have, 
\begin{align}
	((g(\tilde{\vy}_v^{(t)},\mD^{(t)})-g(\vy_v^{(t)},\mD^{(t)}))^\top(\tilde{\vy}_v-\vy_v)
	&\geq \mu\ltwo{\tilde{\vy}_v^{(t)}-\vy_v^{(t)}}^2.
\end{align}
Thus, 
\begin{align*}
\ltwo{\hat{\vx}_v^{(t+1)}-\vx_v^{(t+1)}}^2 & \leq (1+(\alpha^{(t)})^2\beta^2-\alpha^{(t)}\mu) \ltwo{\tilde{\vy}_v^{(t)}-\vy_v^{(t)}}^2.	
\end{align*}
Now, {let $0<\alpha^{(t)}<\mu/\beta^2$} and let $\gamma^{(t)}=\sqrt{1+(\alpha^{(t)})^2\beta^2-\alpha^{(t)}\mu}$. Note that the choice of $\alpha^{(t)}$ ensures that there exists $\gamma<1$ s.t. $\gamma^{(t)}<\gamma$ for all $t$. Thus, we have, 

\begin{align}
	\lVert \hat{\mX}^{(t+1)}-\mX^{(t+1)}\rVert_F^2& = \sum_{v\in\cV } \ltwo{\hat{\vx}_v^{(t+1)}-\vx_v^{(t+1)}}^2\\
	& \leq (\gamma^{(t)})^2 \sum_{v\in\cV} \ltwo{\tilde{\vy}_v^{(t)}-\vy_v^{(t)}}^2\\
	& =  (\gamma^{(t)})^2 \lVert (\tilde{\mX}^{(t)}-\mX^{(t)})\mW\rVert_F^2\\
	& \leq (\gamma^{(t)})^2 \lVert \tilde{\mX}^{(t)}-\mX^{(t)}\rVert_F^2.
\end{align}
 Here, the last step follows from noting that $\mW$ is doubly stochastic and the magnitude of its eigenvalues are bounded by $1$. Therefore,
 
 \begin{align}
\lVert \hat{\mX}^{(t+1)}-\mX^{(t+1)}\rVert_F
& \leq   \gamma^{(t)} \lVert \tilde{\mX}^{(t)}-\mX^{(t)}\rVert_F\\
& \leq \gamma^{(t)}\lVert \hat{\mX}^{(t)}-\mX^{(t)}\rVert_F+\gamma^{(t)} \lVert \mZ^{(t)}\rVert_F.\label{eq:advpert}
 \end{align}
 Now, recursively applying~\eqref{eq:advpert}, we obtain 
 
 \begin{align}
\lVert \hat{\mX}^{(t+1)}-\mX^{(t+1)}\rVert_F
& \leq   \sum_{\tau=0}^t \left(\prod_{\sigma=\tau}^{t}\gamma^{(\sigma)} \right)\lVert \mZ^{(\tau)}\rVert_F.
 \end{align}
 In particular, we have, for some $\gamma\in(0,1)$, 
\begin{align}
 \lVert \hat{\mX}^{(t+1)}-\mX^{(t+1)}\rVert_F
& \leq   \sum_{\tau=0}^t \gamma^{t-\tau+1} \lVert \mZ^{(\tau)}\rVert_F.
 \end{align}

\end{proof}

\subsection{Proof of Theorem~\ref{thm:optimality}}
The proof follows from a ``benign attack'' that shows that no validated learning protocol with success probability approaching $1$ can attain loss values lower than those attained at admissible consensus vectors. The attack hinges on Byzantine agents picking a valid probability distribution (that is not necessarily the true distribution) and executing the protocol as honest agents would. See Protocol~\ref{alg:benignattack} for details.
\begin{algorithm}
    \SetKwInOut{KwIn}{Input}
    \SetKwInOut{KwOut}{Output}
  \KwIn{Agent set $\cV$; Validated Learning protocol $\Pi$ over $T$ rounds with mini-batch size $K$; Byzantine agents $\cB$; Byzantine data distribution $Q_{\cB}$; Honest agents' datasets $(\mD_{v}^{(t)}:v\in\cV\setminus\cB,t\in[T])$}
 	\For{$t=1,2,\ldots,T$}{
   		\For{$v\in\cV\setminus\cB$}{ 
    		Run round $t$ of $\Pi$ at node $v$ with data mini-batch $\mD_v^{(t)}$\tcp{This is run by honest agents}
    		}
    	\For{$v\in\cB$}{
    		Generate $\tilde{D}_{v,k}^{(t)}\sim Q_v, k=1,2,\ldots,K$\\
    		Run round $t$ of $\Pi$ at node $v$ with data mini-batch $\tilde{\mD}_v^{(t)}$\tcp{This is run by Byzantine agents}	
    	}
    	    	
    }
\caption{A Benign Attack}\label{alg:benignattack}
\end{algorithm}
\subsection{Graph model for experiment}\label{graph_model}
We consider an undirected graph with $n=20$ agents. The graph consists of two fully connected subgraphs of 10 agents each. These two subgraphs are connected with two edges chosen randomly. This ensures that there is a benign path between any two benign agents, provided there is at most one adversary. The graph is shown in Fig. \ref{fig:graph}

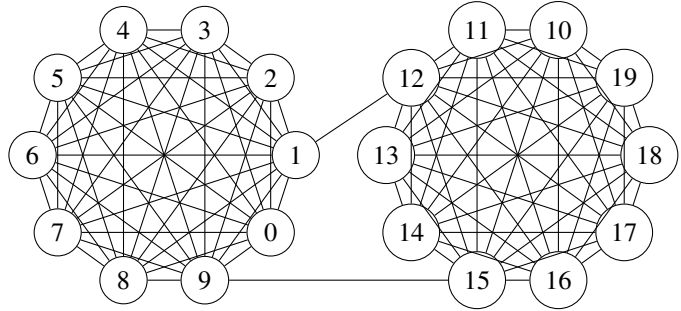
\begin{figure}[h]
    \centering
    \begin{tikzpicture}
    \foreach \x in {0,...,9} {
        \draw ({(\x-1)*36}:1.75cm) node[circle, draw] (\x) {\x};
    }
    \foreach \x in {0,...,9} {
        \foreach \y in {\x,...,9} {
            \ifnum\x=\y
            \else
                \draw (\x) -- (\y);
            \fi
        }
    }
    
    \begin{scope}[xshift=4.7cm]
        \foreach \x in {10,...,19} {
            \draw ({(\x-8)*36}:1.75cm) node[circle, draw] (\x) {\x};
        }
        \foreach \x in {10,...,19} {
            \foreach \y in {\x,...,19} {
                \ifnum\x=\y
                \else
                    \draw (\x) -- (\y);
                \fi
            }
        }
    \end{scope}

    \draw (1) -- (12);
    \draw (9) -- (15);
\end{tikzpicture}
    \caption{Graph of $n=20$ nodes used in the experiments.}
    \label{fig:graph}
\end{figure}
\begin{figure}[h]
    \centering
    \includegraphics[width=0.7\linewidth]{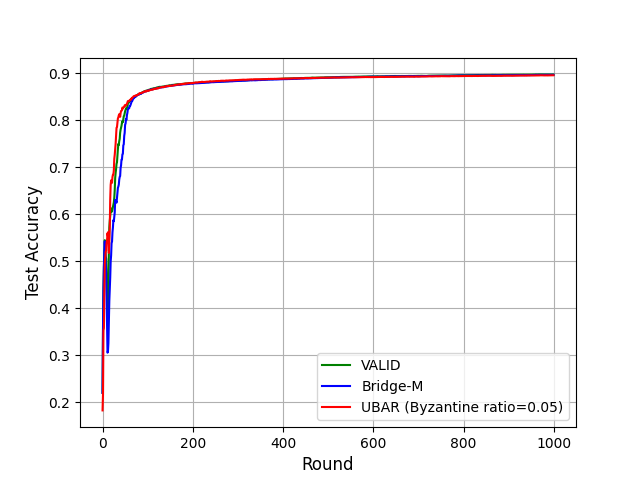}
    \caption{Performance of \protocol compared with UBAR and Bridge-M under i.i.d. setting.}
    \label{fig:iid}
\end{figure}\subsection{Additional experiments}\label{iid_result}
Fig. \ref{fig:iid} compares the performances of \protocol, UBAR, and Bridge-M under the i.i.d setting with no adversarial agents. The result shows approximately the same performance for all methods when there is no adversary.

\end{document}